\documentclass[twoside,10pt]{article}

\usepackage{jmlr2e}

\jmlrheading{17}{2016}{1-21}{3/15; Revised 7/16}{8/16}{Gergely Neu and G\'abor Bart\'ok}

\usepackage{bm}
\usepackage{bbm}

\usepackage{xspace}
\usepackage{amssymb}
\usepackage{amsmath}
\usepackage{graphicx}
\usepackage[ruled]{algorithm2e}

\newcommand{\loss}{\ell}
\newcommand{\hloss}{\wh{\ell}}
\newcommand{\tloss}{\wt{\ell}}

\newcommand{\real}{\mathbb{R}}
\newcommand{\Sw}{\mathcal{S}}

\newcommand{\II}[1]{\mathbbm{1}_{\left\{#1\right\}}}
\newcommand{\PP}[1]{\mathbb{P}\left[#1\right]}
\newcommand{\EE}[1]{\mathbb{E}\left[#1\right]}

\newcommand{\PPc}[2]{\mathbb{P}\left[#1\left|#2\right.\right]}

\newcommand{\PPcc}[2]{\mathbb{P}\left[\left.#1\right|#2\right]}

\newcommand{\EEc}[2]{\mathbb{E}\left[#1\left|#2\right.\right]}

\newcommand{\EEcc}[2]{\mathbb{E}\left[\left.#1\right|#2\right]}

\newcommand{\ev}[1]{\left\{#1\right\}}
\newcommand{\pa}[1]{\left(#1\right)}
\newcommand{\bpa}[1]{\bigl(#1\bigr)}
\newcommand{\Bpa}[1]{\Bigl(#1\Bigr)}
\newcommand{\F}{\mathcal{F}}

\newcommand{\tp}{\wt{p}}
\newcommand{\tq}{\wt{q}}

\renewcommand{\th}{\ensuremath{^{\mathrm{th}}}}
\def\argmin{\mathop{\rm arg\, min}}

\newcommand{\hR}{\wh{R}}
\newcommand{\tL}{\wt{L}}
\newcommand{\p}{p}
\newcommand{\transpose}{^\mathsf{\scriptscriptstyle T}}

\newcommand{\bX}{\bm{X}}
\newcommand{\bV}{\bm{V}}
\newcommand{\br}{\bm{r}}
\newcommand{\be}{\bm{e}}
\newcommand{\bK}{\bm{K}}

\newcommand{\bu}{\bm{u}}

\newcommand{\bU}{\bm{U}}
\newcommand{\bv}{\bm{v}}

\newcommand{\bz}{\bm{z}}
\newcommand{\bloss}{\bm\ell}

\newcommand{\bL}{\bm{L}}

\newcommand{\tbZ}{\widetilde{\bZ}}
\newcommand{\tV}{\widetilde{V}}
\newcommand{\tbV}{\widetilde{\bV}}

\newcommand{\hbl}{\wh{\bloss}}

\newcommand{\hbL}{\wh{\bL}}
\newcommand{\tbL}{\wt{\bL}}
\newcommand{\tbl}{\wt{\bloss}}

\newcommand{\bp}{\bm{p}}
\newcommand{\bZ}{\bm{Z}}

\newcommand{\var}{{\rm Var}}
\newcommand{\varcc}[2]{\var\left[\left.#1\right|#2\right]}
\newcommand{\wh}{\widehat}
\newcommand{\wt}{\widetilde}

\newcommand{\fpl}{\texttt{FPL}\xspace}
\newcommand{\fplrw}{\texttt{FPL+GR}\xspace}
\newcommand{\fplrwp}{\texttt{FPL+GR.P}\xspace}
\newcommand{\gr}{\texttt{GR}\xspace}
\newcommand{\rw}{\texttt{GR}\xspace}
\newcommand{\exph}{\texttt{Exp3}\xspace}
\newcommand{\green}{\texttt{Green}\xspace}
\newcommand{\hedge}{\texttt{Hedge}\xspace}
\newcommand{\ewa}{\texttt{EWA}\xspace}
\newcommand{\osmd}{\texttt{OSMD}\xspace}
\newcommand{\INF}{\texttt{INF}\xspace}

\newcommand{\norm}[2]{\left\|#1\right\|_{#2}}
\newcommand{\infnorm}[1]{\norm{#1}{\infty}}
\newcommand{\onenorm}[1]{\norm{#1}{1}}

\makeatletter
\def\blfootnote{\gdef\@thefnmark{}\@footnotetext}
\makeatother

\ShortHeadings{Importance Weighting Without Importance Weights}{Neu and Bart\'ok}
\firstpageno{1}

\begin{document}

\title{Importance Weighting Without Importance Weights: 
\\An Efficient Algorithm for Combinatorial Semi-Bandits}

\author{\name Gergely Neu \email gergely.neu@gmail.com \\
       \addr Universitat Pompeu Fabra
       \\ Roc Boronat 138, 08018, Barcelona, Spain
       \AND
       \name G\'abor Bart\'ok \email bartok@google.com \\
       \addr Google Z\"urich
       \\ Brandschenkestrasse 100, 8002, Z\"urich, Switzerland}

\editor{Manfred Warmuth}

\maketitle

\begin{abstract}
We propose a sample-efficient alternative for importance weighting for situations where one only has
sample access to the probability distribution that generates the observations. Our new method, called Geometric Resampling (\gr), is 
described and analyzed in the context of online combinatorial optimization under semi-bandit
feedback, where a learner sequentially selects its actions from a combinatorial decision set so as to
minimize its cumulative loss. 
In particular, we show that the well-known Follow-the-Perturbed-Leader (\fpl) prediction method coupled with Geometric Resampling yields the 
first computationally efficient reduction from offline to online optimization in this setting. We
provide a thorough theoretical analysis for the resulting algorithm, showing that its performance is on par with
previous, inefficient solutions. Our main contribution is showing that, despite the relatively large variance induced by
the \rw procedure, our performance guarantees hold with high probability rather than only in expectation. As a side
result, we also improve the best known regret  bounds for \fpl in online combinatorial optimization with full feedback,
closing the perceived performance gap between \fpl and exponential weights in this setting.
\blfootnote{A preliminary version of this paper was published as \citet{NeuBartok13}. Parts of this work were completed while 
Gergely Neu was with the SequeL team at INRIA Lille -- Nord Europe, France and G\'abor Bart\'ok was with the Department of Computer 
Science at ETH Z\"urich.}
\end{abstract}

\begin{keywords}
  online learning, combinatorial optimization, bandit problems, semi-bandit feedback, follow the perturbed leader, importance weighting
\end{keywords}

\section{Introduction}

Importance weighting is a crucially important tool used in many areas of machine learning, and specifically online
learning with partial feedback.  While most work assumes that importance weights are readily available or can be
computed with little effort during runtime, this is often not the case in many practical settings, even when one has
cheap sample access to the distribution generating the observations. Among other cases, such situations may arise when
observations are generated by complex hierarchical sampling schemes, probabilistic programs, or, more
generally, black-box generative models. In this paper, we propose a simple and efficient sampling scheme called
\emph{Geometric Resampling} (\gr) to compute reliable estimates of importance weights \emph{using only sample access}. 

Our main motivation is studying a specific online learning algorithm whose practical applicability in partial-feedback
settings had long been hindered by the problem outlined above. Specifically, we consider the well-known
\emph{Follow-the-Perturbed-Leader} (\fpl) prediction method that maintains implicit sampling distributions that
usually cannot be expressed in closed form. In this paper, we endow \fpl with our Geometric Resampling scheme
to construct the first known computationally efficient reduction from offline to online combinatorial optimization
under an important partial-information scheme known as \emph{semi-bandit feedback}. In the rest of this section, we
describe our precise setting, present related work and outline our main results.

\subsection{Online Combinatorial Optimization}
We consider a special case of online linear optimization known as online combinatorial optimization (see
Figure~\ref{fig:protocol}). In every round $t=1,2,\dots,T$ of this sequential decision problem, the learner chooses
an \emph{action} $\bV_t$ from the finite action set $\Sw\subseteq\ev{0,1}^d$, where $\left\|\bv\right\|_1\le m$ holds
for all $\bv\in\Sw$. At the same time, the environment fixes a loss vector $\bloss_t\in[0,1]^d$ and the learner suffers
loss $\bV_t\transpose\bloss_t$. The goal of the learner is to minimize the cumulative loss
$\sum_{t=1}^T\bV_t\transpose\bloss_t$.
As usual in the literature of online optimization (\citealp{CBLu06:book}), we measure the performance of the learner
in terms of the \emph{regret} defined as
\begin{equation}\label{eq:regret}
R_T= \max_{\bv\in\Sw} \sum_{t=1}^T\left(\bV_t - \bv\right)\transpose\bloss_t = \sum_{t=1}^T\bV_t\transpose\bloss_t -
\min_{\bv\in\Sw} \sum_{t=1}^T\bv\transpose\bloss_t\,,
\end{equation}
that is, the gap between the total loss of the learning algorithm and the best fixed decision in hindsight.
In the current paper, we focus on the case of \emph{non-oblivious} (or \emph{adaptive}) environments, where we allow the
loss vector $\bloss_t$ to depend on the previous decisions $\bV_1,\dots,\bV_{t-1}$ in an arbitrary fashion. Since it is
well-known that no deterministic algorithm can achieve sublinear regret under such weak assumptions, we will consider
learning algorithms that choose their decisions in a randomized way. For such learners, another performance measure
that we will study is the \emph{expected regret} defined as
\[
 \hR_T= \max_{\bv\in\Sw} \sum_{t=1}^T\EE{\left(\bV_t - \bv\right)\transpose\bloss_t} =
\EE{\sum_{t=1}^T\bV_t\transpose\bloss_t} -
\min_{\bv\in\Sw} \EE{\sum_{t=1}^T\bv\transpose\bloss_t}.
\]
\begin{figure}[t]
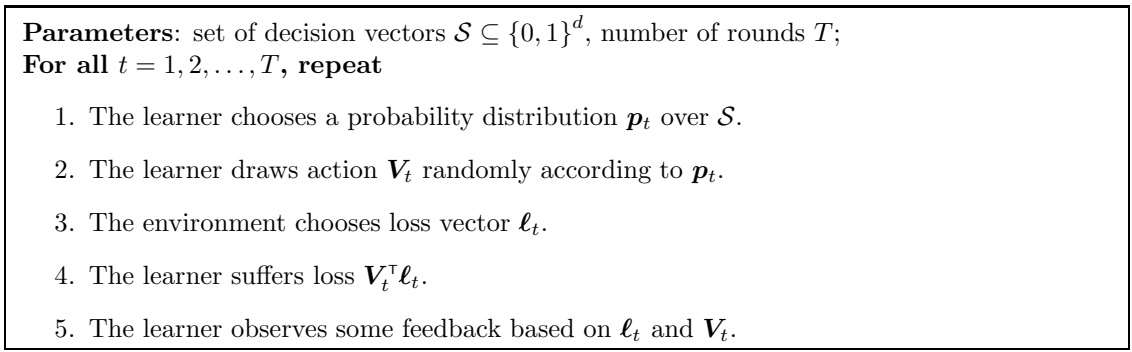

\centering
\fbox{
\begin{minipage}{.95\textwidth}
{\bfseries Parameters}: set of decision vectors $\Sw \subseteq\ev{0,1}^d$,
number of rounds $T$; \\
{\bfseries For all $t=1,2,\dots,T$, repeat}
\begin{enumerate}
\item The learner chooses a probability distribution $\bp_t$ over $\Sw$.
\item The learner draws action $\bV_t$ randomly according to $\bp_t$.
\item The environment chooses loss vector $\bloss_{t}$.
\item The learner suffers loss $\bV_t\transpose\bloss_{t}$.
\item The learner observes some feedback based on $\bloss_t$ and $\bV_t$.
\end{enumerate}
\end{minipage}
}
\caption{The protocol of online combinatorial optimization.}
\label{fig:protocol}
\end{figure}

The framework described above is general enough to accommodate a number of interesting problem instances such as path
planning, ranking and matching problems, finding minimum-weight spanning trees and cut sets. Accordingly, different
versions of this general learning problem have drawn considerable attention in the past few years. These versions differ
in the amount of information made available to the learner after each round $t$. In the simplest setting, called the
\emph{full-information} setting, it is assumed that the learner gets to observe the loss vector $\bloss_t$ regardless of
the choice of $\bV_t$. As this assumption does not hold for many practical applications, it is more interesting
to study the problem under \emph{partial-information} constraints, meaning that the learner only gets some limited
feedback based on its own decision. 
In the current paper, we focus on a more realistic partial-information scheme known as \emph{semi-bandit feedback}
\citep*{audibert13regret} where the learner only observes the components $\loss_{t,i}$ of the loss vector for which
$V_{t,i}=1$, that is, the losses associated with the components selected by the learner.\footnote{Here, $V_{t,i}$ and
$\loss_{t,i}$ are the $i\th$ components of the vectors $\bV_t$ and $\bloss_t$, respectively.}

\subsection{Related Work}
The most well-known instance of our problem is the  \emph{multi-armed bandit} problem considered in the
seminal paper of \citet*{auer2002bandit}: in each round of this problem, the learner has to select one of $N$
\emph{arms} and minimize regret against the best fixed arm while only observing the losses of the chosen arms. In our
framework, this setting corresponds to setting $d=N$ and $m=1$.
Among other contributions concerning this problem, \citeauthor{auer2002bandit}~propose an algorithm called
\exph\,(Exploration and Exploitation using Exponential weights) based on constructing loss estimates $\hloss_{t,i}$ for
each component of the loss vector and playing arm $i$ with probability proportional to $\exp(-\eta \sum_{s=1}^{t-1}
\hloss_{s,i})$ at time $t$, where $\eta>0$ is a parameter of the algorithm, usually called the learning rate\footnote{In
fact, \citeauthor{auer2002bandit}~mix the resulting distribution with a uniform distribution over the arms with
probability $\eta N$. However, this modification is not  needed when one is concerned with the total expected regret,
see, e.g., \citet[Section~3.1]{bubeck12survey}.}. This algorithm is essentially a variant of the Exponentially Weighted
Average (\ewa) forecaster (a variant of weighted majority algorithm of \citealp{LW94}, and aggregating strategies of
\citealp{Vov90}, also known as \hedge by \citealp{FrSc97}). 
Besides proving that the \emph{expected} regret of \exph is $O\bpa{\sqrt{NT\log N}}$,
\citeauthor{auer2002bandit}~also provide a general lower bound of $\Omega\bpa{\sqrt{NT}}$ on the regret of any learning
algorithm on this particular problem. This lower bound was later matched by a variant of the Implicitly Normalized
Forecaster (\INF) of \citet{AB10} by using the same loss estimates in a more refined way. \citeauthor{AB10} also show
bounds of $O\bpa{\sqrt{NT/\log N}\log(N/\delta)}$ on the regret that hold with probability at least $1-\delta$,
uniformly for any $\delta>0$.

The most popular example of online learning problems with actual combinatorial structure is the shortest path problem
first considered by \citet{TW03} in the full information scheme.
The same problem was considered by \citet*{gyorgy07sp}, who proposed an algorithm that works with semi-bandit
information. Since then, we have come a long way in understanding the ``price of information'' in online combinatorial
optimization---see \citet*{audibert13regret} for a complete overview of results concerning all of the information
schemes considered in the current paper. The first algorithm directly targeting general online combinatorial
optimization problems is due to \citet*{KWK10}: their method named \texttt{Component Hedge} guarantees an optimal regret
of
$O\bpa{m\sqrt{T\log (d/m)}}$ in the full
information setting. As later shown by \citet*{audibert13regret}, this algorithm is an instance of a more general algorithm class known as 
Online Stochastic Mirror Descent (\osmd). Taking the idea one step further, \citet*{audibert13regret} also show that
\osmd-based methods can also be used for proving expected regret bounds of $O\bpa{\sqrt{mdT}}$ for the semi-bandit
setting, which is also shown to coincide with the minimax regret in this setting. 
For completeness, we note that the \ewa forecaster is known to attain an
expected regret of $O\bpa{m^{3/2}\sqrt{T\log (d/m)}}$ in the full information case and $O\bpa{m\sqrt{dT\log (d/m)}}$ in
the
semi-bandit case.

While the results outlined above might suggest that there is absolutely no work left to be done in the full information
and semi-bandit schemes, we get a different picture if we restrict our attention to \emph{computationally efficient}
algorithms. First, note that  methods based on exponential weighting of each decision vector can only be efficiently
implemented
for a handful of decision sets $\Sw$---see \citet{KWK10} and \citet{CL12} for some examples. Furthermore, as noted by
\citet{audibert13regret}, \osmd-type methods can be efficiently implemented by convex programming if the convex hull
of the decision set can be described by a polynomial number of constraints. Details of such an efficient implementation
are worked out by \citet*{suehiro12submodular}, whose algorithm runs in $O(d^6)$ time, which can still be prohibitive in
practical applications. While \citet{KWK10} list some further examples where \osmd can be implemented efficiently, we
conclude that there is no general efficient algorithm with near-optimal performance guarantees for learning in
combinatorial semi-bandits.

The Follow-the-Perturbed-Leader (\fpl) prediction method (first proposed by \citealp{Han57} and later rediscovered by
\citealp{KV05}) offers a computationally efficient solution for the online combinatorial optimization problem given
that the \emph{static} combinatorial optimization problem $\min_{\bv\in\Sw} \bv\transpose\bloss$ admits computationally
efficient solutions for any $\bloss\in\real^{d}$. The idea underlying \fpl is very simple: in every round $t$, the
learner draws some random perturbations $\bZ_t\in\real^d$ and selects the action that minimizes the perturbed total
losses:
\[
\bV_t = \argmin_{\bv\in\Sw} \bv\transpose \left(\sum_{s=1}^{t-1}\bloss_s - \bZ_t\right).
\]
Despite its conceptual simplicity and computational efficiency, \fpl have been relatively overlooked until very
recently, due to two main reasons:
\begin{itemize}
\item The best known bound for \fpl in the full information setting is $O\bpa{m\sqrt{dT}}$, which is worse than the
bounds for both \ewa and \osmd that scale only logarithmically with $d$.
\item Considering bandit information, no efficient \fpl-style algorithm is known to achieve a regret of
$O\bpa{\sqrt{T}}$.
On one hand, it is relatively straightforward to prove $O\bpa{T^{2/3}}$ bounds on the expected regret for an efficient
\fpl-variant (see, e.g., \citealp{AweKlein04}  and \citealp{McMaBlu04}). \citet{Pol05} proved
bounds of $O\bpa{\sqrt{NT\log N}}$ in the $N$-armed bandit setting, however, the proposed algorithm requires
$O\bpa{T^2}$ numerical operations per round. 
\end{itemize}
The main obstacle for constructing a computationally efficient \fpl-variant that works with partial information is
precisely the lack of closed-form expressions for importance weights. In the current paper, we address the above two
issues and show that an efficient \fpl-based algorithm using independent exponentially distributed perturbations can achieve as good 
performance guarantees as \ewa in online combinatorial optimization.

Our work contributes to a new wave of positive results concerning \fpl. Besides the reservations towards \fpl
mentioned above, the reputation of \fpl has been also suffering from the fact that the nature of regularization arising
from perturbations is not as well-understood as the explicit regularization schemes underlying \osmd or \ewa. Very
recently, \citet{ALTS14} have shown that \fpl implements a form of strongly convex regularization over the convex hull
of the decision space. Furthermore, \citet{rakhlin12rr} showed that \fpl run with a specific perturbation scheme can be
regarded as a relaxation of the minimax algorithm. Another recently initiated line of work shows that intuitive
\emph{parameter-free} variants of \fpl can achieve excellent performance in full-information settings
(\citealp{devroye13rwalk} and \citealp{EWK14}).

\subsection{Our Results}
In this paper, we propose a loss-estimation scheme called Geometric Resampling to efficiently
compute importance weights for the observed components of the loss vector. Building on this technique and the \fpl
principle, resulting in \emph{an efficient algorithm for regret minimization under semi-bandit feedback}. Besides this
contribution, our techniques also enable us to improve the best known regret bounds of \fpl in the full information
case. We prove the following results concerning variants of our algorithm:
\begin{itemize}
 \item a bound of $O\bigl(m\sqrt{dT\log(d/m)}\bigr)$ on the expected regret under semi-bandit feedback
(Theorem~\ref{thm:bandit_exp}),
 \item a bound of $O\bigl(m\sqrt{dT\log(d/m)} + \sqrt{mdT}\log(1/\delta)\bigr)$ on the regret that holds with
probability at least $1-\delta$, uniformly for all $\delta\in(0,1)$ under semi-bandit feedback
(Theorem~\ref{thm:highprob}),
 \item a bound of $O\bigl(m^{3/2}\sqrt{T\log(d/m)}\bigr)$ on the expected regret under full information
(Theorem~\ref{thm:fullinfo}).
\end{itemize}
We also show that both of our semi-bandit algorithms access the optimization oracle $O(dT)$ times over $T$ rounds with
high probability, increasing the running time only by a factor of $d$ compared to the full-information variant. Notably,
our results close the gaps between the performance bounds of \fpl and \ewa under both full information and semi-bandit
feedback. Table~\ref{tab:table} puts our newly proven regret bounds into context.

\begin{table}[h]
\begin{center}
\begin{tabular}{l|l|l|l}
 & \fpl & \ewa & \osmd\\
 \hline
Full info regret bound & $\mathbf{m^{3/2}\sqrt{T\, log\frac dm}}$ & $m^{3/2}\sqrt{T\log\frac dm}$ & 
$m\sqrt{T\log\frac dm}$ \\
Semi-bandit regret bound & $\mathbf{m\sqrt{dT\, log\frac dm}}$ & $m\sqrt{dT\log\frac dm}$ & 
$\sqrt{mdT}$ \\
Computationally efficient? & always & sometimes & sometimes
\end{tabular}
\end{center}
\caption{Upper bounds on the regret of various algorithms for online combinatorial optimization, up to constant factors. The third row 
roughly describes the computational efficiency of each algorithm---see the text for details. New results are presented in 
boldface.}\label{tab:table}
\end{table}

\section{Geometric Resampling}\label{sec:rw}
In this section, we introduce the main idea underlying Geometric Resampling in the specific context of $N$-armed
bandits where $d=N$, $m=1$ and the learner has access to the basis vectors $\ev{\be_i}_{i=1}^d$ as its decision set
$\Sw$. In this setting, components of the decision vector are referred
to as \emph{arms}.
For ease of notation, define $I_t$ as the unique arm such that $V_{t,I_t} = 1$ and $\F_{t-1}$ as the sigma-algebra
induced by the learner's actions and observations up to the end of round $t-1$. Using this notation, we define
 $p_{t,i} = \PPcc{I_t=i}{\F_{t-1}}$.

Most bandit algorithms rely on feeding some loss estimates to a sequential prediction algorithm.
It is commonplace to consider \emph{importance-weighted} loss estimates of the form
\begin{equation}\label{eq:oldest}
\hloss^*_{t,i} = \frac{\II{I_t=i}}{\p_{t,i}} \loss_{t,i}
\end{equation}
for all $t,i$ such that $p_{t,i}>0$. It is straightforward to show that $\hloss^*_{t,i}$ is an unbiased estimate of
the loss $\loss_{t,i}$ for all such $t,i$. Otherwise, when $p_{t,i}=0$, we set $\hloss_{t,i}^*=0$, which gives
$\EEcc{\hloss_{t,i}^*}{\F_{t-1}} = 0 \le \loss_{t,i}$.

To our knowledge, all existing bandit algorithms operating in the non-stochastic setting utilize some version of the
importance-weighted loss estimates described above. 
This is a very natural choice for algorithms that operate by first computing the probabilities $p_{t,i}$ and then
sampling $I_t$ from the resulting distributions. While many algorithms fall into this class (including the \exph
algorithm of \citet{auer2002bandit}, the \green algorithm of \citet{allenberg06hannan} and the \INF algorithm of
\citet{AB10}, one can think of many other algorithms where the distribution $\bp_t$ is specified implicitly and thus
importance weights are not readily available. Arguably, \fpl is the most important online prediction algorithm that
operates with implicit distributions that are notoriously difficult to compute in closed form. To overcome this
difficulty, we propose a different loss estimate that can be efficiently computed \emph{even when $\bp_{t}$ is not
available for the learner}.

Our estimation procedure dubbed Geometric Resampling (\gr) is based on the simple observation that, even though
$p_{t,I_t}$ might not be computable in closed form, one can simply generate a geometric random variable with
expectation $1/p_{t,I_t}$ by repeated sampling from $\bp_t$. Specifically, we propose the following procedure
to be executed in round $t$:

\vspace{.25cm}
\makebox[\textwidth][c]{
\fbox{
\begin{minipage}{.6\textwidth}
\textbf{Geometric Resampling for multi-armed bandits}
\vspace{.1cm}
\hrule
\begin{enumerate}
\item The learner draws $I_t\sim \bp_t$.
\item For $k=1,2,\dots$
\begin{enumerate}
\item Draw $I'_t(k) \sim \bp_t$.
\item If $I'_t(k) = I_t$, break.
\end{enumerate}
\item Let $K_t = k$.
\end{enumerate}
\end{minipage}
}
}
\vspace{.25cm}

\noindent Observe that $K_t$ generated this way is a geometrically distributed random variable given $I_t$ and
$\F_{t-1}$. Consequently, we
have $\EEc{K_t}{\F_{t-1},I_t} = 1/\p_{t,I_t}$.
We use this property to construct the estimates
\begin{equation}\label{eq:newest}
\hloss_{t,i} = K_t \II{I_t=i} \loss_{t,i} 
\end{equation}
for all arms $i$.
We can easily show that the above estimate is unbiased whenever $p_{t,i}>0$:
\[
\begin{split}
\EEcc{\hloss_{t,i}}{\F_{t-1}} &= \sum_{j} p_{t,j} \EEcc{\hloss_{t,i}}{\F_{t-1},I_t=j}
\\
&= p_{t,i} \EEc{\loss_{t,i} K_t}{\F_{t-1},I_t=i}
\\
&= p_{t,i} \loss_{t,i} \EEc{K_t}{\F_{t-1},I_t=i}
\\
&= \loss_{t,i}.
\end{split}
\]
Notice that the above procedure produces $\hloss_{t,i} = 0$ almost surely whenever $p_{t,i} = 0$, giving
$\EEcc{\hloss_{t,i}}{\F_{t-1}} = 0$ for such $t,i$.

One practical concern with the above sampling procedure is that its worst-case running time is unbounded: while the
expected number of necessary samples $K_t$ is clearly $N$, the actual number of samples might be much larger. 
In the next section, we offer a remedy to this problem, as well as generalize the approach to work in the
combinatorial semi-bandit case.

\section{An Efficient Algorithm for Combinatorial Semi-Bandits}\label{sec:alg}
In this section, we present our main result: an efficient reduction from offline to online combinatorial optimization
under semi-bandit feedback. The most critical element in our technique is extending the Geometric Resampling idea to
the case of combinatorial action sets. For defining the procedure, let us assume that we are running a randomized
algorithm mapping histories to probability distributions over the action set $\Sw$: letting $\F_{t-1}$ denote the
sigma-algebra induced by the history of interaction between the learner and the environment, the algorithm
picks action $\bv\in\Sw$ with probability $p_{t}(\bv) = \PPc{\bV_t = \bv}{\F_{t-1}}$. Also introducing $q_{t,i} =
\EEc{V_{t,i}}{\F_{t-1}}$, we can define the counterpart of the standard importance-weighted loss estimates
of Equation~\ref{eq:oldest} as the vector $\hbl^*_t$ with components
\begin{equation}\label{eq:combest_old}
 \hloss_{t,i}^* = \frac{V_{t,i}}{q_{t,i}} \loss_{t,i}.
\end{equation}
Again, the problem with these estimates is that for many algorithms of practical interest, the importance weights
$q_{t,i}$ cannot be computed in closed form. We now extend the Geometric Resampling procedure defined in the previous
section to estimate the importance weights in an efficient manner. One adjustment we make to the procedure presented in the previous 
section is capping off the number of samples at some finite $M>0$. While this capping obviously introduces some bias, we will show later 
that for appropriate values of $M$, this bias does not hurt the performance of the overall learning algorithm too much. Thus, we define the 
Geometric Resampling procedure for combinatorial semi-bandits as follows:

\vspace{.25cm}
\makebox[\textwidth][c]{
\fbox{
\begin{minipage}{.7\textwidth}
\textbf{Geometric Resampling for combinatorial semi-bandits}
\vspace{.1cm}
\hrule
\begin{enumerate}
\item The learner draws $\bV_t\sim \bp_t$.
\item For $k=1,2,\dots,M$, draw $\bV'_t(k) \sim \bp_t$.
\item For $i=1,2,\dots,d$,
\[
K_{t,i} = \min\bigl(\ev{k: V_{t,i}'(k) = 1}\cup\ev{M}\bigr).
\]
\end{enumerate}
\end{minipage}
}
}
\vspace{.25cm}

\noindent Based on the random variables output by the \rw procedure, we construct our loss-estimate vector $\hbl_t\in\real^d$ with 
components
\begin{align}\label{eq:combest}
\hloss_{t,i} = K_{t,i} V_{t,i} \loss_{t,i}
\end{align}
for all $i=1,2,\dots,d$.
Since $V_{t,i}$ are nonzero only for coordinates for which $\loss_{t,i}$ is observed, these estimates are well-defined.
It also follows that the sampling procedure can be terminated once for every $i$ with $V_{t,i}=1$, there is a copy
$\bV_{t}'(k)$ such that $V_{t,i}'(k)=1$. 

Now everything is ready to define our algorithm:  \fplrw, standing for Follow-the-Perturbed-Leader with
Geometric Resampling. Defining $\hbL_t = \sum_{s=1}^t \hbl_s$, at time step $t$ \fplrw draws
the components of the perturbation vector $\bZ_t$ independently from a standard exponential distribution and selects
action\footnote{By the definition of the perturbation distribution, the minimum is unique almost surely.}
\begin{equation}\label{eq:fpl}
\bV_t = \argmin_{\bv\in\Sw} \bv\transpose \left(\eta\hbL_{t-1} - \bZ_t\right),
\end{equation}
where $\eta>0$ is a parameter of the algorithm. 
As we mentioned earlier, the distribution $\bp_t$, while implicitly specified by $\bZ_t$ and the estimated
cumulative losses $\hbL_{t-1}$, cannot usually be expressed in closed form for \fpl.\footnote{One notable exception is
when the perturbations are drawn independently from standard Gumbel distributions, and the decision set is the
$d$-dimensional simplex: in this case, \fpl is known to be equivalent with \ewa---see, e.g., \citet{ALTS14} for
further discussion.}
However, sampling the actions $\bV_t'(\cdot)$ can be carried out by drawing additional perturbation vectors
$\bZ_t'(\cdot)$ independently from
the same distribution as $\bZ_t$ and then solving a linear optimization task. 
We emphasize that the above additional actions are \emph{never actually played by the algorithm}, but are only necessary
for constructing the loss estimates. The power of \fplrw is that, unlike other algorithms for combinatorial
semi-bandits, its implementation only requires access to a linear optimization oracle over $\Sw$. We point the reader to
Section~\ref{sec:runningtime} for a more detailed discussion of the running time of \fplrw. 
Pseudocode for \fplrw is shown on as Algorithm~\ref{alg:FPLGR}.

\begin{algorithm}[t]
\caption{\fplrw implemented with a waiting list. The notation $\bm{a}\circ\bm{b}$ stands for elementwise product of 
vectors $\bm{a}$ and $\bm{b}$: $(\bm{a}\circ\bm{b})_i = a_ib_i$ for all $i$.}\label{alg:FPLGR}
\textbf{Input}: $\Sw \subseteq\ev{0,1}^d$,
 $\eta\in\real^+$, $M\in\mathbb{Z}^{+}$\;
  \textbf{Initialization}: $\hbL=\boldsymbol{0}\in\real^d$\;
  \For{t=1,\dots,T}
  {
  Draw $\bZ\in\real^d$ with independent components $Z_i \sim \text{Exp}(1)$\;
  Choose action $\bV = \displaystyle\argmin_{\bv\in\Sw} \ev{\bv\transpose \left(\eta\hbL - \bZ\right)}$;
  \tcc*[f]{Follow the perturbed leader}\\
  $\bK=0$; $\br = \bV$; \tcc*[f]{Initialize waiting list and counters}\\
  \For(\tcc*[f]{Geometric Resampling}){k=1,\dots ,M}
  {
    $\bK = \bK + \br$; \tcc*[f]{Increment counter}\\
    Draw $\bZ'\in\real^d$ with independent components $Z'_i \sim \text{Exp}(1)$\;
    $\bV'=\displaystyle\argmin_{\bv\in\Sw} \ev{\bv\transpose \left(\eta\hbL - \bZ'\right)}$; \tcc*[f]{Sample a copy of $\bV$}\\
    $\br = \br \circ \bV'$; \tcc*[f]{Update waiting list}
    \\\lIf(\tcc*[f]{All indices recurred}){$\br=0$}{break}
  }
  $\hbL=\hbL+\bK\circ \bV \circ\bloss$; \tcc*[f]{Update cumulative loss estimates}
  }
\end{algorithm}
As we will show shortly, \fplrw as defined above comes with strong performance guarantees that hold \emph{in
expectation}. One can think of several possible ways to robustify \fplrw so that it provides bounds that hold with high
probability. 
One possible path is to follow \citet{auer2002bandit} and define the loss-estimate vector $\tbl_t^*$ with components
\[
\tloss_{t,i}^* = \hloss_{t,i} - \frac{\beta}{q_{t,i}}
\]
for some $\beta >0$.
The obvious problem with this definition is that it requires perfect knowledge of the importance weights $q_{t,i}$ for
all $i$. While it is possible to extend Geometric Resampling 
developed in the previous sections to construct a reliable proxy to the above loss estimate, there are several
downsides to this approach. First, observe that one would need to obtain estimates of $1/q_{t,i}$ for every single $i$---even for the ones 
for which $V_{t,i} = 0$. Due to this necessity, there is no hope to terminate the sampling procedure in reasonable time. Second, reliable 
estimation requires multiple samples of $K_{t,i}$, where the sample size has to explicitly depend on the desired confidence level.

Thus, we follow a different path: Motivated by the work of \citet{AB10}, we propose to use a
loss-estimate vector $\tbl_t$ with components of the form
\begin{equation}\label{eq:est_hp} 
 \tloss_{t,i} = \frac{1}{\beta} \log\pa{1 + \beta \hloss_{t,i}}
\end{equation}
with an appropriately chosen $\beta>0$.
Then, defining $\tbL_{t-1} = \sum_{s=1}^{t-1} \tbl_s$, we propose a variant of \fplrw that simply replaces $\hbL_{t-1}$
by $\tbL_{t-1}$ in the rule \eqref{eq:fpl} for choosing $\bV_t$.
We refer to this variant of \fplrw as \fplrwp. In the next section, we provide performance guarantees for both
algorithms.

\subsection{Performance Guarantees}
Now we are ready to state our main results. Proofs will be presented in Section~\ref{sec:analysis}. First, we present a
performance guarantee for \fplrw in terms of the
\emph{expected regret}:
\begin{theorem}\label{thm:bandit_exp}
The expected regret of \fplrw satisfies
\[
\hR_T \le \frac{m\left(\log \pa{d/m} +1\right)}{\eta} +
2\eta mdT + \frac{dT}{eM}
\]
under semi-bandit information. In particular, with
\[
\eta = \sqrt{\frac{\log (d/m)+1}{2dT}} \qquad\mbox{and}\qquad M= \left\lceil\frac{\sqrt{dT}}{em\sqrt{2\left(\log
(d/m)+1\right)}}\right\rceil,
\]
the expected regret of \fplrw is bounded as
\[
\hR_T \le 3m\sqrt{2dT\left(\log \frac dm+1\right)}.
\]
\end{theorem}
Our second main contribution is the following bound on the regret of \fplrwp.
\begin{theorem}
 Fix an arbitrary $\delta>0$. With probability at least $1-\delta$, the regret of \fplrwp satisfies
 \[
 \begin{split}
  R_T \le& 
  \frac{m\pa{\log (d/m) + 1}}{\eta} + \eta \pa{Mm\sqrt{2 T\log\frac{5}{\delta}} + 2md\sqrt{T\log \frac{5}{\delta}} + 2mdT} + \frac 
{dT}{eM}
  \\
  & + \beta \pa{M\sqrt{2 m T\log\frac{5}{\delta}} + 2d\sqrt{T\log \frac{5}{\delta}} + 2dT} + \frac{m\log(5d/\delta)}{\beta}
\\
& + m\sqrt{2(e-2)T}\log\frac{5}{\delta} + \sqrt{8T \log \frac {5}{\delta}} + \sqrt{2(e-2)T}.
 \end{split}  
 \]
 In particular, with
 \[
 M = \left\lceil\sqrt{\frac{dT}{m}}\right\rceil,\quad\beta = \sqrt{\frac{m}{dT}},\quad
\mbox{and}\quad \eta =
\sqrt{\frac{\log(d/m) + 1}{dT}},
 \]
the regret of \fplrwp is bounded as
\[
 \begin{split}
  R_T \le& 
  3m\sqrt{dT\pa{\log \frac dm + 1}} + 
\sqrt{mdT} \pa{\log \frac{5d}{\delta} + 2} + \sqrt{2mT\log \frac{5}{\delta}} \pa{\sqrt{\log \frac{d}{m} +1}+1}
\\
& + 1.2m\sqrt{T}\log\frac{5}{\delta} + \sqrt{T} \pa{\sqrt{8\log \frac{5}{\delta}} + 1.2} + 
2\sqrt{d\log\frac{5}{\delta}} \pa{m\sqrt{\log \frac{d}{m} + 1} + \sqrt{m}}
 \end{split}
 \]
 with probability at least $1-\delta$. 
 \label{thm:highprob}
\end{theorem}

\subsection{Running Time}\label{sec:runningtime}
Let us now turn our attention to computational issues.
First, we note that the efficiency of \fpl-type algorithms crucially depends on the availability of an efficient oracle
that solves the static combinatorial optimization problem of finding $\argmin_{\bv\in\Sw} \bv\transpose \bloss$.
Computing
the running time of the full-information variant of \fpl is straightforward: assuming that the oracle computes the
solution to the static problem in $O(f(\Sw))$ time,
 \fpl returns its prediction in $O(f(\Sw) + d)$ time (with the $d$ overhead coming from the time necessary to
generate the perturbations). Naturally, our loss estimation scheme multiplies these computations by the number of
samples taken in each round. While terminating the estimation procedure after $M$ samples helps in controlling the
running time with high probability, observe that the na\"ive bound of $MT$ on the number of samples becomes way too
large when setting $M$ as suggested by Theorems~\ref{thm:bandit_exp} and~\ref{thm:highprob}.
The next proposition shows that the amortized running time of Geometric Resampling remains as low as $O(d)$ even  for
large values of $M$.
\begin{proposition}
  Let $S_t$ denote the number of sample actions taken by \rw in round $t$. Then, $\EE{S_t} \le d$.
  Also, for any $\delta>0$, 
  \[
   \sum_{t=1}^T S_t \le (e-1) dT + M\log\frac 1\delta
  \]
  holds with probability at least $1-\delta$.
\end{proposition}
\begin{proof}
For proving the first statement, let us fix a time step $t$ and notice that
\[
 S_t = \max_{j:V_{t,j}=1} K_{t,j}= \max_{j=1,2,\dots,d} V_{t,j} K_{t,j} \le \sum_{j=1}^d V_{t,j} K_{t,j}.
\]
Now, observe that $\EEcc{K_{t,j}}{\F_{t-1}, V_{t,j}} \le 1/\EEcc{V_{t,j}}{\F_{t-1}}$, which gives
$\EE{S_t}\le d$, thus proving the first statement.
For the second part, notice that $X_t = \pa{S_t - \EEcc{S_t}{\F_{t-1}}}$ is a martingale-difference sequence with
respect to $\pa{\F_t}$ with $X_t\le M$ and with conditional variance
\[
\begin{split}
 \varcc{X_t}{\F_{t-1}} &= \EEcc{\pa{S_t - \EEcc{S_t}{\F_{t-1}}}^2}{\F_{t-1}} \le 
 \EEcc{S_t^2}{\F_{t-1}}
 \\
 &= \EEcc{\max_j \pa{V_{t,j} K_{t,j}}^2}{\F_{t-1}}
 \le \EEcc{\sum_{j=1}^d V_{t,j} K_{t,j}^2}{\F_{t-1}} 
 \\
 &\le \sum_{j=1}^d \min\ev{\frac{2}{q_{t,j}},M} \le dM,
\end{split}
\]
where we used $\EEcc{K_{t,i}^2}{\F_{t-1}} = \frac{2-q_{t,i}}{q_{t,i}^2}$.
Then, the second statement follows from applying a version of Freedman's inequality due to \citet{BLLRS11} (stated as
Lemma~\ref{lem:mart} in the appendix) with $B = M$ and $\Sigma_T \le dMT$.
\end{proof}
Notice that choosing $M = O\bpa{\sqrt{dT}}$ as suggested by Theorems~\ref{thm:bandit_exp} and~\ref{thm:highprob}, the
above
result  guarantees that the amortized running time of \fplrw is $O\big((d + \sqrt{d/T})\cdot(f(\Sw)+d)\big)$ with high
probability.

\section{Analysis}\label{sec:analysis}
This section presents the proofs of Theorems~\ref{thm:bandit_exp} and~\ref{thm:highprob}. In a didactic attempt, we
present statements concerning the loss-estimation procedure and the learning algorithm separately:
Section~\ref{sec:analysis_rw} presents various important properties of the loss estimates produced by Geometric Resampling, 
Section~\ref{sec:analysis_fpl} presents general tools for analyzing Follow-the-Perturbed-Leader methods.
Finally, Sections~\ref{sec:analysis_exp} and~\ref{sec:analysis_hp} put these results together to prove
Theorems~\ref{thm:bandit_exp} and~\ref{thm:highprob}, respectively.
\subsection{Properties of Geometric Resampling}\label{sec:analysis_rw}
The basic idea underlying Geometric Resampling is replacing the importance weights $1/q_{t,i}$ by appropriately defined
random variables $K_{t,i}$. As we have seen earlier (Section~\ref{sec:rw}), running \rw with $M=\infty$ amounts to
sampling each $K_{t,i}$
from a geometric distribution with expectation $1/q_{t,i}$, yielding an unbiased loss estimate. In practice, one would
want to set $M$ to a finite value to ensure that the running time of the sampling procedure is bounded. Note however
that early termination of \rw introduces a bias in the loss estimates. This section is mainly concerned with the nature
of this bias. We emphasize that the statements presented in this section remain valid no matter what randomized
algorithm generates the actions $\bV_t$. Our first lemma gives an explicit expression on the expectation of the loss
estimates generated by \rw.
\begin{lemma}\label{lem:bias}
For all $j$ and $t$ such that $q_{t,j}>0$, the loss estimates \eqref{eq:combest}  satisfy
\[
\EEcc{\hloss_{t,j}}{\F_{t-1}} = \left(1-(1-q_{t,j})^M \right) \loss_{t,j}.
\]
\end{lemma}
\begin{proof}
Fix any $j,t$ satisfying the condition of the lemma. 
Setting $q = q_{t,j}$ for simplicity, we write
\[
\begin{split}
\EEc{K_{t,j}}{\F_{t-1}} =&
\sum_{k=1}^\infty k (1-q)^{k-1} q - \sum_{k=M}^\infty (k - M) (1-q)^{k-1} q
\\
=&
\sum_{k=1}^\infty k (1-q)^{k-1} q - (1-q)^M \sum_{k=M}^\infty (k\! - \!M) (1-q)^{k\!-\!M\!-\!1} q
\\
=&
\left(1-(1-q)^M \right)\sum_{k=1}^\infty k (1-q)^{k-1} q = \frac{1-(1-q)^M }{q}.
\end{split}
\]
The proof is concluded by combining the above with $\EEcc{\hloss_{t,j}}{\F_{t-1}} \!=\! q_{t,j} \loss_{t,j}
\EEcc{K_{t,j}}{\F_{t-1}}$.
\end{proof}
The following lemma shows two important properties of the \gr loss estimates \eqref{eq:combest}.
Roughly speaking, the first of these properties ensure that any learning algorithm relying on these estimates will be
\emph{optimistic} in the sense that the loss of any \emph{fixed} decision will be underestimated in expectation. The
second property ensures that the learner will not be \emph{overly optimistic} concerning its own performance.
\begin{lemma}\label{lem:rwprops}
 For all $\bv\in\Sw$ and $t$, the loss estimates \eqref{eq:combest} satisfy the following two properties:
 \begin{eqnarray}
  \EEcc{\bv\transpose\hbl_t}{\F_{t-1}} &\le& \bv\transpose\bloss_t,\\
  \EEcc{\sum_{\bu\in\Sw}p_t(\bu) \pa{\bu\transpose\hbl_t}}{\F_{t-1}} &\ge& 
  \sum_{\bu\in\Sw}p_t(\bu) \bpa{\bu\transpose\bloss_t} - \frac{d}{eM}.
 \end{eqnarray}
\end{lemma}
\begin{proof}
 Fix any $\bv\in\Sw$ and $t$. 
 The first property is an immediate consequence of Lemma~\ref{lem:bias}: we have that $\EEcc{\hloss_{t,k}}{\F_{t-1}}\le
\loss_{t,k}$ for all $k$, and thus $\EEcc{\bv\transpose\hbl_t}{\F_{t-1}} \le \bv\transpose \bloss_t$. 
For the second statement, observe that 
\[
\begin{split}
\EEcc{\sum_{\bu\in\Sw}p_t(\bu) \pa{\bu\transpose\hbl_t }}{\F_{t-1}} &= 
\sum_{i=1}^d q_{t,i} \EEcc{\hloss_{t,i}}{\F_{t-1}}
= \sum_{i=1}^d q_{t,i} \left(1-(1-q_{t,i})^M \right) \loss_{t,i}
\end{split}
\]
also holds by Lemma~\ref{lem:bias}. 
To control the bias term $\sum_i q_{t,i} (1-q_{t,i})^M$, note that $q_{t,i} (1-q_{t,i})^M \le q_{t,i} e^{-Mq_{t,i}}$.
By elementary calculations, one can show that $f(q) = qe^{-Mq}$ takes its maximum at $q=\frac 1M$ and thus
$
\sum_{i=1}^d q_{t,i} (1-q_{t,i})^M \le \frac{d}{eM}.
$
\end{proof}
Our last lemma concerning the loss estimates \eqref{eq:combest} bounds the conditional variance of the estimated
loss of the learner. This term plays a key role in the performance analysis of \exph-style algorithms (see,
e.g., \citet{auer2002bandit,UNK10,audibert13regret}), as well as in the analysis presented in the current paper.
\begin{lemma}\label{lem:quad}
 For all $t$, the loss estimates \eqref{eq:combest} satisfy 
 \[
  \EEcc{\sum_{\bu\in\Sw}p_t(\bu) \pa{\bu\transpose\hbl_t}^2}{\F_{t-1}} \le 2md.
 \]
\end{lemma}
Before proving the statement, we remark that the conditional variance can be bounded as $md$ for the
standard (although usually infeasible) loss estimates \eqref{eq:combest_old}. That is, the above lemma shows that,
somewhat surprisingly, the variance of our estimates is only twice as large as the variance of the standard estimates.
\begin{proof}
Fix an arbitrary $t$. For simplifying notation below, let us introduce $\tbV$ as an independent copy of $\bV_t$ such
that $\PPcc{\tbV=\bv}{\F_{t-1}} = p_t(\bv)$ holds for all $\bv\in\Sw$. To begin, observe that for any $i$
\begin{equation}\label{eq:K2bound}
\EEcc{K_{t,i}^2}{\F_{t-1}} \le \frac{2-q_{t,i}}{q_{t,i}^2} \le \frac{2}{q_{t,i}^2}
\end{equation}
holds, as $K_{t,i}$ has a truncated geometric law.
The statement is proven as 
\[
\begin{split}
\EEcc{\sum_{\bu\in\Sw}p_t(\bu) \pa{\bu\transpose\hbl_t}^2}{\F_{t-1}}
&=
\EEcc{\sum_{i=1}^d\sum_{j=1}^d \left(\wt{V}_{i}\hloss_{t,i}\right)\left(\wt{V}_{j}\hloss_{t,j}\right)}{\F_{t-1}}
\\
&=
\EEcc{\sum_{i=1}^d\sum_{j=1}^d \left(\wt{V}_{i}K_{t,i}V_{t,i}\loss_{t,i}\right)
\left(\wt{V}_{j}K_{t,j}V_{t,j}\loss_{t,j}\right)}{\F_{t-1}}
\\
&\qquad\qquad\mbox{(using the definition of $\hbl_t$)}
\\
&\le
\EEcc{\sum_{i=1}^d\sum_{j=1}^d
\frac{K_{t,i}^2+K_{t,j}^2}{2}\left(\wt{V}_{i}V_{t,i}\loss_{t,i}\right)\left(\wt{V}_{j}V_{t,j}\loss_{t,j}\right)}
{\F_{t-1}}
\\
&\qquad\qquad\mbox{(using $2AB\le A^2 + B^2$)}
\\
&\le
2\EEcc{\sum_{i=1}^d\frac{1}{q_{t,i}^2}\left(\wt{V}_{i}V_{t,i}\loss_{t,i}\right)
\sum_{j=1}^d V_{t,j}\loss_{t,j}}{\F_{t-1}}
\\
&\qquad\qquad\mbox{(using symmetry, Eq.~\eqref{eq:K2bound} and $\wt{V}_j\le 1$)}
\\
&\le
2m\EEcc{\sum_{j=1}^d\loss_{t,j}}{\F_{t-1}} \le 2md,
\end{split}
\]
where the last line follows from using $\onenorm{\bV_t}\le m$, $\infnorm{\bloss_t}\le 1$, and
$\EEcc{V_{t,i}}{\F_{t-1}}=\EEcc{\wt{V}_{i}}{\F_{t-1}}=q_{t,i}$.
\end{proof}

\subsection{General Tools for Analyzing \fpl}\label{sec:analysis_fpl}
In this section, we present the key tools for analyzing the \fpl-component of our learning algorithm.
In some respect, our analysis is a synthesis of previous work on \fpl-style methods: we borrow several ideas
from \citet{Pol05} and the proof of Corollary~4.5 in \citet{CBLu06:book}. Nevertheless, our analysis is the first  to
directly target combinatorial settings, and yields the tightest known bounds for \fpl in this domain. Indeed, the tools
developed in this section also permit an improvement for \fpl in the full-information setting, closing the presumed
performance gap between \fpl and \ewa in both the full-information and the semi-bandit settings. The statements we
present in this section are not specific to the loss-estimate vectors used by \fplrw.

Like most other known work, we study the performance of the learning algorithm through a \emph{virtual algorithm} that
(\emph{i}) uses a time-independent perturbation vector and (\emph{ii})~is~allowed to peek one step into the future.
Specifically,
letting $\tbZ$ be a perturbation vector drawn independently from the same distribution as
$\bZ_1$, the virtual algorithm picks its $t\th$ action as
\begin{equation}\label{eq:cheater}
\tbV_t = \argmin_{\bv\in\Sw} \ev{\bv\transpose\left(\eta \hbL_{t} - \tbZ\right)}.
\end{equation}
In what follows, we will crucially use that $\tbV_t$ and $\bV_{t+1}$ {are conditionally
independent and identically distributed given $\F_{t}$}. In particular, 
introducing the notations
\begin{align*}
q_{t,i} &= \EEcc{V_{t,i}}{\F_{t-1}}  & \tq_{t,i} &= \EEcc{\tV_{t,i}}{\F_{t}}\\
p_{t}(\bv) &= \PPcc{\bV_t=\bv}{\F_{t-1}}    & \tp_{t}(\bv) &= \PPcc{\tbV_t=\bv}{\F_{t}},
\end{align*}
we will exploit the above property by using $q_{t,i} = \tq_{t-1,i}$ and $p_{t}(\bv) = \tp_{t-1}(\bv)$ numerous times in
the proofs below.

First, we show a regret bound on the virtual algorithm that plays the action sequence $\tbV_1,\tbV_{2},\dots,\tbV_T$.
\begin{lemma}\label{lem:cheat}
For any $\bv\in\Sw$,
 \begin{equation}
\begin{split}
\sum_{t=1}^T \sum_{\bu\in\Sw} \tp_t(\bu)\pa{\left(\bu - \bv\right)\transpose \hbl_t} \le \frac{m\left(\log
(d/m)+1\right)}{\eta}.
\end{split}
\end{equation}
\end{lemma}
Although the proof of this statement is rather standard, we include it for completeness. We also note that the lemma
slightly improves other known results by replacing the usual $\log d$ term by $\log(d/m)$.
\begin{proof}
 Fix any $\bv\in\Sw$. Using Lemma~3.1 of \citet{CBLu06:book} (sometimes referred to as the
{\sl ``follow-the-leader/be-the-leader''} lemma) for the sequence $\bpa{\eta\hbl_1 -
\tbZ,\allowbreak\eta\hbl_2,\allowbreak\dots,\allowbreak\eta\hbl_T}$, we
obtain
\[
\eta \sum_{t=1}^T \tbV_t\transpose \hbl_t - \tbV_1\transpose \tbZ \le \eta \sum_{t=1}^T \bv\transpose \hbl_t -
\bv\transpose \tbZ.
\]
Reordering and integrating both sides with respect to the distribution of $\tbZ$ gives
\begin{equation}
\begin{split}
\eta \sum_{t=1}^T \sum_{\bu\in\Sw} \tp_t(\bu) \pa{\left(\bu - \bv\right)\transpose \hbl_t }\le \EE{\left(\tbV_1 -
\bv\right)\transpose\tbZ}.
\end{split}
\end{equation}
The statement follows from using $\EE{\tbV_1\transpose \tbZ} \le m(\log(d/m) + 1)$, which is proven in
Appendix~\ref{app:proofs} as Lemma~\ref{lem:expmax}, noting that $\tbV_1\transpose \tbZ$ is upper-bounded by the sum of the $m$ 
largest components of $\tbZ$.
\end{proof}
The next lemma relates the performance of the virtual algorithm to the actual performance of \fpl.  The lemma relies on
a ``sparse-loss'' trick similar to the trick used in the proof Corollary~4.5 in \citet{CBLu06:book}, and is also
related to the ``unit rule'' discussed by \citet{KWK10}.
\begin{lemma}\label{lem:price}
 For all $t=1,2,\dots,T$, assume that $\hbl_t$ is such that $\hloss_{t,k}\ge 0$ for all $k\in\ev{1,2,\dots,d}$. Then,
\[
 \sum_{\bu\in\Sw} \bigl(p_t(\bu) - \tp_t(\bu)\bigr) \pa{\bu\transpose \hbl_t }\le \eta \sum_{\bu\in\Sw} p_t(\bu)
\left(\bu\transpose \hbl_{t}\right)^2.
\] 
\end{lemma}
\begin{proof}
Fix an arbitrary $t$ and $\bu\in\Sw$, and define the ``sparse loss vector'' $\hbl^-_t(\bu)$ with components
$\hloss^-_{t,k}(\bu) = u_k \hloss_{t,k}$ and
\[
\bV^-_t(\bu) = \argmin_{\bv\in\Sw} \ev{\bv\transpose\left(\eta\hbL_{t-1} + \eta\hbl^-_t(\bu) - \tbZ\right)}.
\]
Using the notation $p^-_{t}(\bu) = \PPcc{\bV^-_{t}(\bu)=\bu}{\F_{t}}$, we show in Lemma~\ref{lem:modified-loss}
(stated and proved in Appendix~\ref{app:proofs}) that $p^-_{t}(\bu)\le \tp_{t}(\bu)$.
Also, define
\[
\bU(\bz) = \argmin_{\bv\in\Sw} \ev{\bv\transpose\left(\eta\hbL_{t-1} - \bz\right)}.
\]
Letting $f(\bz) = e^{-\onenorm{\bz}}$ $(\bz\in\real_+^d)$ be the density of the perturbations, we  have
\[
\begin{split}
p_{t}(\bu)
&=\int\limits_{\bz\in[0,\infty]^d} \II{\bU(\bz)=\bu} f(\bz) \,d\bz
\\
&= e^{\eta \left\|\hbl^-_{t}(\bu)\right\|_1} \int\limits_{\bz\in[0,\infty]^d} \II{\bU(\bz)=\bu}
f\left(\bz+\eta\hbl^-_t(\bu)\right) 
\,d\bz
\\
&= e^{\eta \left\|\hbl^-_{t}(\bu)\right\|_1} \idotsint\limits_{z_i\in[\hloss^-_{t,i}(\bu),\infty]}
\II{\bU\left(\bz-\eta\hbl^-_t(\bu)\right)=\bu} f(\bz)  \,d\bz
\\
&\le e^{\eta \left\|\hbl^-_{t}(\bu)\right\|_1} \int\limits_{\bz\in[0,\infty]^d}
\II{\bU\left(\bz-\eta\hbl^-_t(\bu)\right)=\bu} f(\bz) 
\,d\bz
\\
&\le e^{\eta \left\|\hbl^-_{t}(\bu)\right\|_1} p^-_{t}(\bu) \le e^{\eta \left\|\hbl^-_{t}(\bu)\right\|_1} \tp_{t}(\bu).
\end{split}
\]
Now notice that $\bigl\|\hbl^-_{t}(\bu)\bigr\|_1 = \bu\transpose \hbl^-_{t}(\bu) = \bu\transpose \hbl_{t}$, which gives
\[
\begin{split}
\tp_{t}(\bu) &\ge p_{t}(\bu) e^{-\eta \bu\transpose \hbl_{t}} \ge p_{t}(\bu) \left(1 - \eta \bu\transpose
\hbl_{t}\right).
\end{split}
\]
The proof is concluded by repeating the same argument for all $\bu\in\Sw$, reordering and summing the terms as
\begin{equation}
\begin{split}\label{eq:pineq}
\sum_{\bu\in\Sw} p_t(\bu) \pa{\bu\transpose \hbl_t}
&\le \sum_{\bu\in\Sw} \tp_{t}(\bu) \pa{\bu\transpose \hbl_{t}} + \eta \sum_{\bu\in\Sw} p_{t}(\bu) \left(\bu\transpose
\hbl_{t}\right)^2.
\end{split}
\end{equation}
\end{proof}

\subsection{Proof of Theorem~\ref{thm:bandit_exp}}\label{sec:analysis_exp}
Now, everything is ready to prove the bound on the expected regret of \fplrw. 
Let us fix an arbitrary $\bv\in\Sw$.
 By putting together Lemmas~\ref{lem:quad}, \ref{lem:cheat} and~\ref{lem:price}, we immediately obtain
 \begin{equation}\label{eq:estbound}
  \EE{\sum_{t=1}^T \sum_{\bu\in\Sw} p_t(\bu)\pa{\left(\bu - \bv\right)\transpose \hbl_t} }\le \frac{m\left(\log
(d/m)+1\right)}{\eta} + 2\eta mdT,
 \end{equation}
 leaving us with the problem of upper bounding the expected regret in terms of the left-hand side of the above
inequality. This can be done by using the properties of the loss estimates \eqref{eq:combest} stated in
Lemma~\ref{lem:rwprops}:
\[
 \EE{\sum_{t=1}^T \left(\bV_{t} - \bv\right)\transpose \bloss_t} \le \EE{\sum_{t=1}^T \sum_{\bu\in\Sw}
p_t(\bu) \pa{\left(\bu - \bv\right)\transpose\hbl_t}} +
\frac{dT}{eM}.
\]
Putting the two inequalities together proves the theorem.

\subsection{Proof of Theorem~\ref{thm:highprob}}\label{sec:analysis_hp}
We now turn to prove a bound on the regret of \fplrwp that holds with high probability. 
We begin by noting that the conditions of Lemmas~\ref{lem:cheat} and~\ref{lem:price} continue to hold for the new loss estimates, 
so we can obtain the central terms in the regret:
\[
 \sum_{t=1}^T \sum_{\bu\in\Sw} \p_t(\bu)\pa{\pa{\bu - \bv}\transpose\tbl_t }\le \frac{m (\log (d/m) + 1)}{\eta} + \eta
\sum_{t=1}^T \sum_{\bu\in\Sw} p_t(\bu)\pa{\bu\transpose\tbl_t}^2.
\]
The first challenge posed by the above expression is relating the left-hand side to the true regret with high
probability. Once this is done, the remaining challenge is to bound the second term on the right-hand side, as well as
the other terms arising from the first step. We first show that the loss estimates used by \fplrwp consistently
underestimate the true losses with high probability.
\begin{lemma}\label{lem:bias_hp}
Fix any $\delta'>0$. For any $\bv\in\Sw$,
\[
\bv\transpose\pa{\tbL_{T} - \bL_{T}} \le  \frac{m\log\pa{d/\delta'}}{\beta}
\]
holds with probability at least $1-\delta'$.
\end{lemma}
The simple proof is directly inspired by Appendix C.9 of \citet{AB10}.
\begin{proof}
 Fix any $t$ and $i$. Then,
 \[
 \begin{split}
  \EEcc{\exp\pa{\beta \tloss_{t,i}}}{\F_{t-1}} = \EEcc{\exp\pa{\log\pa{1+\beta\hloss_{t,i}}}}{\F_{t-1}} \le 1 +
\beta\loss_{t,i} \le
\exp(\beta \loss_{t,i}),
 \end{split}
 \]
 where we used Lemma~\ref{lem:bias} in the first inequality and $1+z \le e^z$ that holds for all $z\in\real$. As a
result, the process $W_t = \exp\Bpa{\beta \bpa{\tL_{t,i} - L_{t,i}}}$ is a supermartingale with respect to $\pa{\F_t}$:
$ \EEcc{W_t}{\F_{t-1}} \le W_{t-1}$.
Observe that, since $W_0 = 1$, this implies $\EE{W_t} \le \EE{W_{t-1}} \le \ldots \le 1$. Applying 
Markov's inequality gives that
\[
\begin{split}
 \PP{\tL_{T,i} > L_{T,i} + \varepsilon} &=
 \PP{\tL_{T,i} - L_{T,i} > \varepsilon}
 \\
 &\le \EE{\exp\pa{\beta \pa{\tL_{T,i} - L_{T,i}}}} \exp(-\beta \varepsilon) \le \exp(-\beta \varepsilon)
\end{split}
\]
holds for any $\varepsilon>0$. The statement of the lemma follows after using
$\onenorm{\bv}\le m$, applying the union bound for all $i$, and solving for $\varepsilon$.
\end{proof}
The following lemma states another key property of the loss estimates.
\begin{lemma}\label{lem:quad2}
 For any $t$,
 \[
  \sum_{i=1}^d q_{t,i} \hloss_{t,i} \le \sum_{i=1}^d q_{t,i} \tloss_{t,i} + \frac \beta2  
\sum_{i=1}^d q_{t,i} \hloss_{t,i}^2.
 \]
\end{lemma}
\begin{proof}
The statement follows trivially from the inequality $\log(1+z) \ge z - \frac{z^2}{2}$ that holds for all $z\ge 0$. In
particular, for any fixed $t$ and $i$, we have
\[
 \log\pa{1 + \beta \hloss_{t,i}} \ge \beta \hloss_{t,i} - \frac{\beta^2}{2} \hloss_{t,i}^2.
\]
Multiplying both sides by $q_{t,i}/\beta$ and summing for all $i$ proves the statement.
\end{proof}
The next lemma relates the total loss of the learner to its total estimated losses.
\begin{lemma}
Fix any $\delta'>0$. With probability at least $1-2\delta'$,
 \[
 \begin{split}
 \sum_{t=1}^T \bV_t\transpose\bloss_t \le& \sum_{t=1}^T \sum_{\bu\in\Sw} p_{t}(\bu) \pa{\bu\transpose \hbl_{t}} + \frac
{dT}{eM} + \sqrt{2(e-2)T}\pa{m\log\frac{1}{\delta'} + 1}  + \sqrt{8T \log \frac {1}{\delta'}}
 \end{split}
 \]
\end{lemma}
\begin{proof}
We start by rewriting
\[
 \begin{split}
 \sum_{\bu\in\Sw} p_{t}(\bu) \pa{\bu\transpose \hbl_{t}} &= \sum_{i=1}^d q_{t,i} K_{t,i} V_{t,i} \loss_{t,i}.
 \end{split}
\]
Now  let $k_{t,i} = \EEcc{K_{t,i}}{\F_{t-1}}$ for all $i$ and notice that
\[
 X_t = \sum_{i=1}^d q_{t,i} V_{t,i}\loss_{t,i} \pa{k_{t,i} - K_{t,i}}
\]
is a martingale-difference sequence with respect to $\pa{\F_{t}}$ with elements upper-bounded by $m$ (as
Lemma~\ref{lem:bias} implies $k_{t,i}
q_{t,i} \le 1$ and $\onenorm{\bV_t}\le m$).
Furthermore, the conditional variance of the increments is bounded as
\[
\begin{split}
 \varcc{X_t}{\F_{t-1}} \le& \EEcc{\pa{\sum_{i=1}^d q_{t,i} V_{t,i}K_{t,i}}^2}{\F_{t-1}}
 \le\EEcc{\sum_{j=1}^d V_{t,j} \pa{\sum_{i=1}^d q_{t,i}^2 K_{t,i}^2}}{\F_{t-1}} \le 2m,
\end{split}
\]
where the second inequality is Cauchy--Schwarz and the third one follows from $\EEcc{K_{t,i}^2}{\F_{t-1}} \le
2/q_{t,i}^2$ and $\onenorm{\bV_t} \le m$. Thus, applying Lemma~\ref{lem:mart} with $B = m$ and $\Sigma_T \le 2mT$
we get that for any $S\ge m\sqrt{\log\frac{1}{\delta'}\big/(e-2)}$,
\[
 \sum_{t=1}^T \sum_{i=1}^d q_{t,i} \loss_{t,i}V_{t,i} \pa{k_{t,i} - K_{t,i}} \le 
 \sqrt{(e-2)\log\frac{1}{\delta'}}\pa{\frac{2mT}{S} + S}
\]
holds with probability at least $1-\delta'$, where we have used $\onenorm{\bV_t}\le m$. After setting $S =
m\sqrt{2T\log\frac{1}{\delta'}}$, we obtain that
\begin{equation}\label{eq:plossbound}
 \sum_{t=1}^T \sum_{i=1}^d q_{t,i} \loss_{t,i}V_{t,i} \pa{k_{t,i} - K_{t,i}} \le 
 \sqrt{2\pa{e-2} T}\pa{m\log\frac{1}{\delta'} + 1}
\end{equation}
holds with probability at least $1-\delta'$.

To proceed, observe that $q_{t,i} k_{t,i} = 1 - (1-q_{t,i})^M$ holds by  Lemma~\ref{lem:bias}, implying
\[
 \sum_{i=1}^d q_{t,i} V_{t,i} \loss_{t,i} k_{t,i} \ge \bV_t\transpose\bloss_t - 
 \sum_{i=1}^d V_{t,i} (1-q_{t,i})^M.
\]
Together with Eq.~\eqref{eq:plossbound}, this gives
\[
\begin{split}
 \sum_{t=1}^T \bV_t\transpose\bloss_t \le& \sum_{t=1}^T\sum_{\bu\in\Sw} p_{t}(\bu) \pa{\bu\transpose \hbl_{t}} +
\sqrt{2\pa{e-2} T}\pa{m\log\frac{1}{\delta'} + 1}
+ \sum_{t=1}^T
\sum_{i=1}^d V_{t,i} (1 - q_{t,i})^M.
 \end{split}
\]
Finally, we use that, by Lemma~\ref{lem:rwprops}, $(1 - q_{t,i})^M \le 1/(eM)$, and 
\[
 Y_t = \sum_{i=1}^d \left(V_{t,i} - q_{t,i}\right) (1 - q_{t,i})^M
\]
is a martingale-difference sequence with respect to $\pa{\F_t}$ with increments bounded in $[-1,1]$. Then, by an
application of Hoeffding--Azuma inequality, we have
\[
  \sum_{t=1}^T \sum_{i=1}^d V_{t,i} (1 - q_{t,i})^M \le \frac{dT}{eM} + \sqrt{8 T \log \frac {1}{\delta'}}
\]
with probability at least $1-\delta'$, thus proving the lemma.
\end{proof}
Finally, our last lemma in this section bounds the second-order terms arising from Lemmas~\ref{lem:price}
and~\ref{lem:quad2}.
\begin{lemma}\label{lem:secondorder}
Fix any $\delta'>0$. With probability at least $1-2\delta'$, the following hold simultaneously:
 \[
 \begin{split}
 \sum_{t=1}^T \sum_{\bv\in\Sw} p_{t}(\bv) \pa{\bv\transpose \hbl_{t}}^2 &\le 
 Mm\sqrt{2 T\log\frac{1}{\delta'}} + 2md\sqrt{T\log \frac{1}{\delta'}} + 2mdT
 \\
 \sum_{t=1}^T \sum_{i=1}^d q_{t,i} \hloss_{t,i}^2 &\le 
 M\sqrt{2 m T\log\frac{1}{\delta'}} + 2d\sqrt{T\log \frac{1}{\delta'}} + 2dT.
 \end{split}
 \]
 \end{lemma}
\begin{proof}
First, recall that 
\[
\EEcc{\sum_{\bv\in\Sw} p_{t}(\bv) \pa{\bv\transpose \hbl_{t}}^2}{\F_{t-1}} \le 2md
\]
holds by Lemma~\ref{lem:price}. Now, observe that
\[
 X_t = \sum_{\bv\in\Sw} p_{t}(\bv) \pa{\pa{\bv\transpose \hbl_{t}}^2  - \EEcc{\pa{\bv\transpose
\hbl_{t}}^2}{\F_{t-1}}}
\]
is a martingale-difference sequence with increments in $[-2md,mM]$. An application of the Hoeffding--Azuma inequality gives that
\[
 \sum_{t=1}^T \sum_{\bv\in\Sw} p_{t}(\bv) \pa{\pa{\bv\transpose \hbl_{t}}^2  - \EEcc{\pa{\bv\transpose
\hbl_{t}}^2}{\F_{t-1}}} \le Mm\sqrt{2 T\log\frac{1}{\delta'}} + 2md\sqrt{T\log \frac{1}{\delta'}}
\]
holds with probability at least $1-\delta'$. Reordering the terms completes the proof of the first statement. The second statement is 
proven analogously, building on the bound
\[
\begin{split}
\EEcc{\sum_{i=1}^d q_{t,i} \hloss_{t,i}^2}{\F_{t-1}} \le&
\EEcc{\sum_{i=1}^d q_{t,i} V_{t,i} K_{t,i}^2}{\F_{t-1}} \le 2d.
\end{split}
\]
\end{proof}
Theorem~\ref{thm:highprob} follows from combining Lemmas~\ref{lem:bias_hp} through~\ref{lem:secondorder} and applying
the union bound.

\section{Improved Bounds for Learning With Full Information}
Our proof techniques presented in Section~\ref{sec:analysis_fpl} also enable us to tighten the guarantees for \fpl in
the full information setting. In particular, consider the algorithm choosing action
\[
\bV_t = \argmin_{\bv\in\Sw} \bv\transpose \left(\eta\bL_{t-1} - \bZ_t\right),
\]
where $\bL_t = \sum_{s=1}^t \bloss_s$ and the components of $\bZ_t$ are drawn independently from a standard exponential
distribution.
We state our improved regret bounds concerning this algorithm in the following theorem.
\begin{theorem}\label{thm:fullinfo}
For any $\bv\in\Sw$, the total expected regret of \fpl satisfies
\[
\hR_T\le \frac{m\left(\log (d/m)+1\right)}{\eta} + \eta m \sum_{t=1}^T
\EE{\bV_t\transpose\bloss_t}
\]
under full information. In particular, defining $L_T^* = \min_{\bv\in\Sw} \bv\transpose L_T$ and setting
\[
\eta = \min\ev{\sqrt{\frac{\log (d/m)+1}{L_T^*}}, \frac 12},
\]
the regret of \fpl satisfies
\[
R_T \le 4m\max\ev{\sqrt{L_T^*\left(\log \pa{\frac dm}+1\right)}, \pa{m^2 + 1} \pa{\log\frac dm + 1}}.
\]
\end{theorem}
In the worst case, the above bound becomes $2m^{3/2}\sqrt{T\bigl(\log(d/m)+1\bigr)}$, which improves the best known
bound for \fpl of \citet{KV05} by a factor of $\sqrt{d/m}$. 
\begin{proof}
The first statement follows from combining Lemmas~\ref{lem:cheat} and~\ref{lem:price}, and bounding
\[
 \sum_{\bu\in\Sw}^N p_{t}(\bu) \bpa{\bu\transpose \bloss_{t}}^2 \le m \sum_{\bu\in\Sw}^N p_{t}(\bu)
\bpa{\bu\transpose \bloss_{t}},
\]
while the second one follows from standard algebraic manipulations.
\end{proof}

\section{Conclusions and Open Problems}
In this paper, we have described the first \emph{general and efficient} algorithm for
online combinatorial optimization under semi-bandit feedback. We have proved that the regret of this algorithm is
$O\bpa{m\sqrt{dT\log (d/m)}}$ in this setting, and have also shown that \fpl can achieve $O\bpa{m^{3/2}\sqrt{T\log
(d/m)}}$ in the full information case when tuned properly. While these bounds are off by a factor of $\sqrt{m\log
(d/m)}$ and $\sqrt{m}$ from the respective minimax results, they exactly match the best known regret bounds for the
well-studied Exponentially Weighted Forecaster (\ewa). Whether the remaining gaps can be closed for \fpl-style
algorithms (e.g., by using more intricate perturbation schemes or a more refined analysis) remains an important open
question. Nevertheless, we regard our contribution as a
significant step towards understanding the inherent trade-offs between computational efficiency and performance
guarantees in online combinatorial optimization and, more generally, in online optimization.

The efficiency of our method rests on a novel loss estimation method called Geometric Resampling (\rw). This
estimation method is not specific to the proposed learning algorithm. While \rw has no immediate benefits for
\osmd-type algorithms where the ideal importance weights are readily available, it is possible to think about
problem instances where \ewa can be efficiently implemented while importance weights are difficult to compute.  

The most important open problem left is the case of efficient online linear optimization with \emph{full bandit
feedback} where the learner only observes the inner product $\bV_t\transpose\bloss_t$ in round $t$.
Learning algorithms for this problem usually require that the (pseudo-)inverse of the covariance matrix $P_t =
\EEcc{\bV_t \bV_t\transpose}{\F_{t-1}}$ is readily available for the learner at each time step (see, e.g.,
\citet{McMaBlu04,DHK08,CL12,BCK12}). Computing this matrix, however, is at least as challenging as computing the
individual importance weights $1/q_{t,i}$. 
That said, our Geometric Resampling technique can be directly generalized to this setting by observing that the matrix
geometric series $\sum_{n=1}^\infty (I-P_t)^n$ converges to $P_t^{-1}$ under certain conditions. 
This sum can then be efficiently estimated by sampling independent copies of $\bV_t$, which paves the path for
constructing low-bias estimates of the loss vectors. While it seems straightforward to go ahead and use these estimates
in tandem with \fpl, we have to note that the analysis presented in this paper does not carry through directly in this
case. The main limitation is that our techniques only apply for loss vectors with \emph{non-negative} elements
(cf.~Lemma~\ref{lem:price}). Nevertheless, we believe that Geometric Resampling should be a crucial component for
constructing truly effective learning algorithms for this important problem.

\newpage
\acks{The authors wish to thank Csaba Szepesv\'ari for thought-provoking discussions. The research presented in this paper was supported by 
the UPFellows Fellowship (Marie Curie COFUND program n${^\circ}$ 
600387), the French Ministry of Higher Education and Research and by FUI project Herm\` es.}

\appendix

\section{Further Proofs and Technical Tools}\label{app:proofs}

\begin{lemma}\label{lem:expmax}
 Let $Z_1,\dots,Z_d$ be i.i.d.~exponentially distributed random variables with unit expectation and let $Z_1^*,\dots,Z_d^*$ be their
permutation such that $Z_1^*\ge Z_2^* \ge \dots\ge Z_d^*$. Then, for any $1\le m\le d$,
\[
 \EE{\sum_{i=1}^m Z_i^*} \le m \pa{\log \pa{\frac{d}{m}} + 1}.
\]
\end{lemma}

\begin{proof}
 Let us define $Y = \sum_{i=1}^m Z_i^*$.
 Then, as $Y$ is nonnegative, we have for any $A\ge 0$ that
 \[
 \begin{split}
  \EE{Y} =& \int_{0}^\infty \PP{Y>y} dy
  \\
  \le& A + \int_{A}^\infty \PP{\sum_{i=1}^m Z_i^*>y} dy
  \\
  \le& A + \int_{A}^\infty \PP{Z_1^*>\frac ym} dy  
  \\
  \le& A + d \int_{A}^\infty \PP{Z_1>\frac ym} dy
  \\
  =& A + d e^{-A/m}
  \\
  \le& m\log\pa{\frac dm} + m,
 \end{split}
 \]
 where in the last step, we used that $A = \log \pa{\frac dm}$ minimizes $A + d e^{-A/m}$ over the real line.
\end{proof}

\begin{lemma}\label{lem:modified-loss}
Fix any $\bv\in\Sw$ and any vectors $\bL\in \real^d$ and $\bloss\in[0,\infty)^d$. Define the
vector $\bloss'$ with components $\loss'_k = v_k\loss_k$. Then, for any perturbation vector $\bZ$ with independent
components,
\[
\begin{split}
&\PP{\bv\transpose\left(\bL + \bloss' - \bZ\right)\le \bu\transpose\left(\bL + \bloss' - \bZ\right)\,
\left(\forall
\bu\in\Sw\right)}
\\
&\qquad\le
\PP{\bv\transpose\left(\bL + \bloss - \bZ\right)\le \bu\transpose\left(\bL + \bloss - \bZ\right)\, \left(\forall
\bu\in\Sw\right)}.
\end{split}
\]
\end{lemma}
\begin{proof}
Fix any $\bu \in \Sw\setminus\ev{\bv}$ and define the vector $\bloss'' = \bloss - \bloss'$.  Define the
events
\[
A'(\bu) = \ev{\bv\transpose\left(\bL + \bloss' - \bZ\right)\le \bu\transpose\left(\bL + \bloss' -
\bZ\right)}
\]
and
\[
A(\bu) = \ev{\bv\transpose\left(\bL + \bloss - \bZ\right)\le \bu\transpose\left(\bL + \bloss - \bZ\right)}.
\]
We have
\[
\begin{split}
A'(\bu) &= \ev{\left(\bv-\bu\right)\transpose\bZ \ge \left(\bv-\bu\right)\transpose\left(\bL+\bloss'\right)}
\\
&\subseteq \ev{\left(\bv-\bu\right)\transpose\bZ \ge
\left(\bv-\bu\right)\transpose\left(\bL+\bloss'\right) - \bu\transpose\bloss''}
\\
&= \ev{\left(\bv-\bu\right)\transpose\bZ \ge \left(\bv-\bu\right)\transpose\left(\bL+\bloss\right)}
= A(\bu),
\end{split}
\]
where we used $\bv\transpose\bloss''=0$ and $\bu\transpose\bloss''\ge 0$. Now, since $A'(\bu)\subseteq A(\bu)$, we have
$
\cap_{\bu\in\Sw} A'(\bu) \subseteq \cap_{\bu\in\Sw} A(\bu)
$,
thus proving
$
\PP{\cap_{\bu\in\Sw} A'(\bu)} \le \PP{\cap_{\bu\in\Sw} A(\bu)}
$
as claimed in the lemma.
\end{proof}

\begin{lemma}[cf.~Theorem~1 in \citet{BLLRS11}]\label{lem:mart}
Assume $X_1,X_2,\dots,X_T$ is a martingale-difference sequence with respect to the filtration $(\F_t)$ with
$X_t\le B$ for $1\le t \le T$. Let
$\sigma_t^2 = \var\left[\left.X_t\right|\mathcal{F}_{t-1}\right]$ and $\Sigma_t^2 = \sum_{s=1}^t \sigma_s^2.$
Then, for any $\delta>0$,
\[
\PP{\sum_{t=1}^T \bX_t > B \log \frac 1\delta + (e-2) \frac{\Sigma_T^2}{B}}\le
\delta.
\]
Furthermore, for any $S>B\sqrt{\log(1/\delta))(e-2)}$,
\[
\PP{\sum_{t=1}^T \bX_t > \sqrt{(e-2)\log\frac{1}{\delta}} \pa{\frac{\Sigma_T^2}{S} + S}}\le
\delta.
\]
\end{lemma}


\begin{thebibliography}{28}
\providecommand{\natexlab}[1]{#1}
\providecommand{\url}[1]{\texttt{#1}}
\expandafter\ifx\csname urlstyle\endcsname\relax
  \providecommand{\doi}[1]{doi: #1}\else
  \providecommand{\doi}{doi: \begingroup \urlstyle{rm}\Url}\fi

\bibitem[Abernethy et~al.(2014)Abernethy, Lee, Sinha, and Tewari]{ALTS14}
J.~Abernethy, C.~Lee, A.~Sinha, and A.~Tewari.
\newblock Online linear optimization via smoothing.
\newblock In \emph{Proceedings of The 27th Conference on Learning Theory ({COLT})}, pages 807--823, 2014.

\bibitem[Allenberg et~al.(2006)Allenberg, Auer, Gy{\"o}rfi, and
  Ottucs{\'a}k]{allenberg06hannan}
C.~Allenberg, P.~Auer, L.~Gy{\"o}rfi, and {\text{Gy}}.~Ottucs{\'a}k.
\newblock Hannan consistency in on-line learning in case of unbounded losses
  under partial monitoring.
\newblock In \emph{Proceedings of the 17th International Conference on Algorithmic
  Learning Theory ({ALT})}, pages 229--243, 2006.

\bibitem[Audibert and Bubeck(2010)]{AB10}
J.-Y. Audibert and S.~Bubeck.
\newblock Regret bounds and minimax policies under partial monitoring.
\newblock \emph{Journal of Machine Learning Research}, 11:\penalty0 2635--2686,
  2010.

\bibitem[Audibert et~al.(2014)Audibert, Bubeck, and Lugosi]{audibert13regret}
J.-Y. Audibert, S.~Bubeck, and G.~Lugosi.
\newblock Regret in online combinatorial optimization.
\newblock \emph{Mathematics of Operations Research}, 39:\penalty0 31--45, 2014.

\bibitem[Auer et~al.(2002)Auer, Cesa-Bianchi, Freund, and
  Schapire]{auer2002bandit}
P.~Auer, N.~Cesa-Bianchi, Y.~Freund, and R.~E. Schapire.
\newblock The nonstochastic multiarmed bandit problem.
\newblock \emph{SIAM Journal on Computing}, 32\penalty0 (1):\penalty0 48--77, 2002.

\bibitem[Awerbuch and Kleinberg(2004)]{AweKlein04}
B.~Awerbuch and R.~D. Kleinberg.
\newblock Adaptive routing with end-to-end feedback: distributed learning and
  geometric approaches.
\newblock In \emph{Proceedings of the 36th Annual {ACM}
  Symposium on Theory of Computing}, pages 45--53, 2004.

\bibitem[Beygelzimer et~al.(2011)Beygelzimer, Langford, Li, Reyzin, and
  Schapire]{BLLRS11}
A.~Beygelzimer, J.~Langford, L.~Li, L.~Reyzin, and R.~E. Schapire.
\newblock Contextual bandit algorithms with supervised learning guarantees.
\newblock In \emph{Proceedings of the 14th International Conference on Artificial Intelligence and Statistics (AISTATS)}, pages
  19--26, 2011.

\bibitem[Bubeck et~al.(2012)Bubeck, Cesa-Bianchi, and Kakade]{BCK12}
S.~Bubeck, N.~Cesa-Bianchi, and S.~M. Kakade.
\newblock Towards minimax policies for online linear optimization with bandit
  feedback. 
\newblock In \emph{Proceedings of The 25th Conference on Learning Theory ({COLT})}, pages 1--14, 2012.

\bibitem[Bubeck and Cesa-Bianchi(2012)]{bubeck12survey}
S.~Bubeck and N.~Cesa-Bianchi.
\newblock \emph{Regret Analysis of Stochastic and Nonstochastic Multi-armed
  Bandit Problems}.
\newblock Now Publishers Inc, 2012.

\bibitem[Cesa-Bianchi and Lugosi(2006)]{CBLu06:book}
N.~Cesa-Bianchi and G.~Lugosi.
\newblock \emph{Prediction, Learning, and Games}.
\newblock Cambridge University Press, New York, NY, USA, 2006.

\bibitem[Cesa-Bianchi and Lugosi(2012)]{CL12}
N.~Cesa-Bianchi and G.~Lugosi.
\newblock Combinatorial bandits.
\newblock \emph{Journal of Computer and System Sciences}, 78:\penalty0
  1404--1422, 2012.

\bibitem[Dani et~al.(2008)Dani, Hayes, and Kakade]{DHK08}
V.~Dani, T.~Hayes, and S.~Kakade.
\newblock {The price of bandit information for online optimization}.
\newblock In \emph{Advances in Neural Information Processing Systems (NIPS)},
  volume~20, pages 345--352, 2008.

\bibitem[Devroye et~al.(2013)Devroye, Lugosi, and Neu]{devroye13rwalk}
L.~Devroye, G.~Lugosi, and G.~Neu.
\newblock Prediction by random-walk perturbation.
\newblock In \emph{Proceedings of the 26th Conference on Learning Theory}, pages 460--473, 2013.

\bibitem[Freund and Schapire(1997)]{FrSc97}
Y.~Freund and R.~Schapire.
\newblock A decision-theoretic generalization of on-line learning and an
  application to boosting.
\newblock \emph{Journal of Computer and System Sciences}, 55:\penalty0
  119--139, 1997.

\bibitem[Gy\"{o}rgy et~al.(2007)Gy\"{o}rgy, Linder, Lugosi, and
  Ottucs\'{a}k]{gyorgy07sp}
A.~Gy\"{o}rgy, T.~Linder, G.~Lugosi, and {\relax Gy}.~Ottucs\'{a}k.
\newblock The on-line shortest path problem under partial monitoring.
\newblock \emph{Journal of Machine Learning Research}, 8:\penalty0 2369--2403,
  2007.

\bibitem[Hannan(1957)]{Han57}
J.~Hannan.
\newblock Approximation to {B}ayes risk in repeated play.
\newblock \emph{Contributions to the Theory of Games}, 3:\penalty0 97--139,
  1957.

\bibitem[Kalai and Vempala(2005)]{KV05}
A.~Kalai and S.~Vempala.
\newblock Efficient algorithms for online decision problems.
\newblock \emph{Journal of Computer and System Sciences}, 71:\penalty0
  291--307, 2005.

\bibitem[Koolen et~al.(2010)Koolen, Warmuth, and Kivinen]{KWK10}
W.~Koolen, M.~Warmuth, and J.~Kivinen.
\newblock Hedging structured concepts.
\newblock In \emph{Proceedings of the 23rd Conference on Learning Theory
  (COLT)}, pages 93--105, 2010.

\bibitem[Littlestone and Warmuth(1994)]{LW94}
N.~Littlestone and M.~Warmuth.
\newblock The weighted majority algorithm.
\newblock \emph{Information and Computation}, 108:\penalty0 212--261, 1994.

\bibitem[McMahan and Blum(2004)]{McMaBlu04}
H.~B. McMahan and A.~Blum.
\newblock Online geometric optimization in the bandit setting against an
  adaptive adversary.
\newblock In \emph{Proceedings of the
  17th Conference on Learning Theory ({COLT})}, pages 109--123, 2004.

\bibitem[Neu and Bart\'ok(2013)]{NeuBartok13}
G.~Neu and G.~Bart\'ok.
\newblock An efficient algorithm for learning with semi-bandit feedback.
\newblock In \emph{Proceedings of the 24th International Conference on Algorithmic
  Learning Theory (ALT)}, pages 234--248, 2013.

\bibitem[Poland(2005)]{Pol05}
J.~Poland.
\newblock {FPL} analysis for adaptive bandits.
\newblock In \emph{3rd Symposium on Stochastic Algorithms, Foundations and
  Applications (SAGA)}, pages 58--69, 2005.

\bibitem[Rakhlin et~al.(2012)Rakhlin, Shamir, and Sridharan]{rakhlin12rr}
S.~Rakhlin, O.~Shamir, and K.~Sridharan.
\newblock Relax and randomize: From value to algorithms.
\newblock In \emph{Advances in Neural Information Processing Systems (NIPS)}, volume 25, pages
  2150--2158, 2012.

\bibitem[Suehiro et~al.(2012)Suehiro, Hatano, Kijima, Takimoto, and
  Nagano]{suehiro12submodular}
D.~Suehiro, K.~Hatano, S.~Kijima, E.~Takimoto, and K.~Nagano.
\newblock Online prediction under submodular constraints.
\newblock In \emph{Proceedings of the 23rd International Conference on Algorithmic
  Learning Theory (ALT)}, pages 260--274, 2012.

\bibitem[Takimoto and Warmuth(2003)]{TW03}
E.~Takimoto and M.~Warmuth.
\newblock Paths kernels and multiplicative updates.
\newblock \emph{Journal of Machine Learning Research}, 4:\penalty0 773--818,
  2003.

\bibitem[Uchiya et~al.(2010)Uchiya, Nakamura, and Kudo]{UNK10}
T.~Uchiya, A.~Nakamura, and M.~Kudo.
\newblock Algorithms for adversarial bandit problems with multiple plays.
\newblock In \emph{Proceedings of the 21st International Conference on
  Algorithmic Learning Theory (ALT)}, pages 375--389, 2010.

\bibitem[{Van Erven} et~al.(2014){Van Erven}, Warmuth, and {Kot\l
  owski}]{EWK14}
T.~{Van Erven}, M.~Warmuth, and W.~{Kot\l owski}.
\newblock Follow the leader with dropout perturbations.
\newblock In \emph{Proceedings of The 27th Conference on Learning Theory (COLT)}, pages 949--974, 2014.

\bibitem[Vovk(1990)]{Vov90}
V.~Vovk.
\newblock Aggregating strategies.
\newblock In \emph{Proceedings of the 3rd Annual Workshop on Computational
  Learning Theory (COLT)}, pages 371--386, 1990.

\end{thebibliography}
\end{document}